%% file: sample_paper.tex
\documentclass[onecolumn,12pt]{IEEEtran}

\usepackage{amsthm}
\usepackage{amssymb}
\usepackage{amsmath}
\usepackage{bbm}
\newtheorem{definition}{Definition}
\newtheorem{theorem}{Theorem}

\newtheorem{lem}{Lemma}
\usepackage{times}
\usepackage{algorithm}
\usepackage{algpseudocode}

\usepackage{setspace}
\begin{document}

%

%

\title{Reinforcement Learning for Mean Field Game}
\author{Mridul Agarwal, Vaneet Aggarwal, Arnob Ghosh, and Nilay Tiwari \thanks{The author names are written in an alphabetical ordering. 
		
M. Agarwal and V. Aggarwal are with Purdue University. The work of A. Ghosh and N. Tiwari was performed when they were at Purdue University. }}


\maketitle
\input{abstract}
\input{intro}
\input{background}
\input{formulation}

\input{algo}

\input{results}
\input{conclusion}

\bibliographystyle{IEEEtran}
\bibliography{refs}

\end{document}

%% file: abstract.tex
\begin{abstract}
Stochastic games provide a framework for interactions among multiple agents and enable a myriad of applications. In these games, agents decide on actions simultaneously, the state of every agent moves to the next state, and each agent receives a reward. However, finding an equilibrium (if exists) in this game is often difficult when the number of agents becomes large. This paper focuses on finding a mean-field equilibrium (MFE) in an action coupled stochastic game setting in an episodic framework. It is assumed that the impact of the other agents' can be assumed by the empirical distribution of the mean of the actions. All agents know the action distribution and employ lower-myopic best response dynamics to choose the optimal oblivious strategy. This paper proposes a posterior sampling based approach for reinforcement learning in the mean-field game, where each agent samples a transition probability from the previous transitions. We show that the policy and action distributions converge to the optimal oblivious strategy and the limiting distribution, respectively, which constitute an MFE. 

\end{abstract}

%% file: intro.tex
\section{Introduction}
We live in a world where multiple agents interact repeatedly in a common environment. For example, multiple robots interact to achieve a specific goal. Multi-agent reinforcement learning (MARL) refers to the problem of learning and planning in a sequential decision-making system with unknown underlying system dynamic. The agents need to learn the system by trying different options and observing their actions. Learning in a MARL is fundamentally different from the traditional single-agent reinforcement learning problem (RL) since agents not only interact with the environment but also with each other. Thus, an agent when tries to learn the underlying system dynamics has to consider the action taken by the other agents. Changes in the policy (or actions) of any agent affect the others and vice versa.

One natural learning algorithm is to extend existing RL algorithms to the MARL by assuming that the other agents’ actions are independent.   However,  studies show that a {\em smart} agent which learns the joint actions of the others performs better as compared to the agent that does not learnt the joint action of other agents \cite{tan1993multi,panait}.  For any agent, the actions of other agents become a part of the state. This results in the state space increase exponentially as the number of agents increase. When the agents are strategic, {\em i.e.}, they only want to take actions which maximize their own utility (or value), Nash equilibrium is often employed as the equilibrium concept. The existing equilibrium solving approaches work for some restricted games when there exists an adversarial equilibrium or coordination equilibrium \cite{littman}. Also, these approaches can handle a handful of agents because of the exponential increase in the state space. The computational complexity of finding Nash Equilibrium at every stage game prevents applications of these approaches in games where the number of agents is large \cite{hu}.

In this paper, we consider MARL for an environment where a large number of agents co-exist. Similar to \cite{lang}, we utilize a mean-field approach where we assume that the Q-function of an agent is affected by the mean actions of the others. Mean-field game drastically reduces the complexity, since an agent only  needs to consider the empirical distribution of the actions played by other agents. Further, we consider an {\em oblivious} strategy \cite{johari}, where each agent takes an action based only on its state. An agent does not have to track the policy evolution of the other agents. Unlike \cite{lang}, we consider a generalized version of the game where the state can be different for different agents. Further, we do not consider a game where the adversarial equilibrium or coordinated equilibrium is required to be present. We also do not need to track the action and the realized Q-value of other agents as was the case in \cite{lang}. 

Such mean-field games exist in several domains. For example, the mean-field game is observed in a {\em security game} where a large number of agents make individual decisions about their own security \cite{miao2019optimal,kolokoltsov2016mean}. However, the ultimate security depends on the decisions made by other agents. For example, in a network of computers, if an agent invests heavily in building firewalls, its computer can still be breached if the agents' computers are not secure. In the security game, each agent invests a certain amount to attain a security level. However, the investment level depends on the investment made by the other agents. If the number of agents is large, the game can be modeled as the mean-field game as the average investment made by per agent impacts the decision of an agent. Another example of a mean-field game is the demand response price in the smart grid \cite{li2015energy,farzaneh2018deterministic}. The utility company sets a price based on the average demand per household. Hence, if at a certain time, the average demand is high, the agent may want to reduce its consumption to decrease the price.

Unlike the standard literature on the mean-field equilibrium on stochastic games \cite{johari,glynn,adlakha}, we consider that the transition probabilities are unknown to the agents. Instead, each agent learns the underlying transition probability matrix a readily implementable posterior sampling approach \cite{agrawal,osband2013more}. 
All agents employ best response dynamics to choose the best response strategy which maximizes the (discounted) payoff for the remaining episode length. We show the asymptotic convergence of the policy and the action distribution to the optimal oblivious strategy and the limiting action distribution respectively. Also, the converged point constitutes a mean-field equilibrium (MFE). We use the compactness of state and action space for the convergence. We estimate the value function using backward induction and show that the value function estimates converge to the optimal value function of the true distribution.

Recently, the authors of \cite{subramanian2019reinforcement} studied a policy-gradient based approach to achieve mean-field equilibrium. In contrast, this paper considers a posterior sampling based approach. The key reason of the choice is that the posterior sampling based approach has guaranteed regret bounds for single-agent learning \cite{osband2013more,agrawal}, which are not there for policy-gradient based approaches. Thus, the proposed work has a potential to lead to analysis of regret bounds in the future. 

The key contribution of the paper is a Posterior-sampling based algorithm that is used by each agent in a multi-agent setting, which is shown to converge to a mean-field equilibrium. The proposed algorithm does not assume the knowledge of transition probabilities, and learns them using a posterior sampling approach.


%% file: background.tex
\section{Background}
\subsection{Stochastic Game}
An $n$-player stochastic game is formalized by the tuple ${\cal M}^*$ = $\{\mathcal{S}, \mathcal{A}, P, r, \tau, \rho, \gamma\}$. The agents are indexed by the set $[n] = \{1, 2, \cdots, n\}$.
The state of the $i^{th}$ agent at time $t$ is given by $s_{i,t} \in \mathcal{S}$, where $\mathcal{S}$ is the state space set. $\mathcal{A}(s)$ is defined as the set of the feasible actions any agent can take in state $s$. $\mathcal{A}$ is the action space set defined as $\bigcup_{s\in \mathcal{S}}\mathcal{A}(s)$. Both $\mathcal{S}$ and $\mathcal{A}$ are finite sets.
The combined state space of the system becomes $\mathcal{S}^{\times n} = \mathcal{S}\times\mathcal{S}\times\cdots\times\mathcal{S}$, and the combined action space of the system becomes $\mathcal{A}^{\times n} = \mathcal{A}\times\mathcal{A}\times\cdots\times\mathcal{A}$.
Let $\mathbf{s}_t = \mathbf{s} \in\mathcal{S}^{\times n}$ be a vector of length $n$ made from states of the $n$ agents at time $t$. Then, if the agents play joint action $\mathbf{a}_t = \mathbf{a} \in \mathcal{A}^{\times n}$ the next state of the system $\mathbf{s}_{t+1} = \mathbf{s'}\in\mathcal{S}^{\times n}$ follows the probability distribution $P(\mathbf{s}_{t+1} = \mathbf{s'}|\mathbf{s}_t = \mathbf{s}, \mathbf{a}_t = \mathbf{s})$. 
Along with the state updates, $i^{th}$ agent also receives a reward $R_{i, t} = r_i(\mathbf{s}_{t}, \mathbf{a}_t, \mathbf{s}_{t+1}) \in [0,1]$. The constant $\gamma\in [0,1)$ is the discount factor, and $\rho$ is the initial state distribution such that $\mathbf{s}_{0} \sim \rho$.

We consider an episodic framework where the length of the time horizon is $\tau$. State space set $\mathcal{S}$, action space set $\mathcal{A}$, $\tau$ are known and need not be learned by the agent. We consider that the game is played in episodes $k = 0, 1, 2,\cdots$. The length of each episode is given by $\tau$. In each episode, the game is played in discrete steps, $j \in [\tau]= \{1, 2,\cdots, \tau\}$. The episodes begin at times $t_k = (k-1)\tau + 1, k=1, 2, 3, \cdots$. At each time $t$, the state of the agent $i$ is given by $s_{i,t}$, the agent selects an action $a_{i,t}$, agent observes a scalar reward $r_{i,t}$ and the state transitions to the state $s_{t+1}$. Let $H_{i,t} = (s_{i,1},a_{i,1},r_{i,1}, \cdots, s_{i,t-1}, a_{i,t-1},r_{i,t-1}, s_{i,t})$ denote the history available to the agent $i$ till time $t$.

%% file: formulation.tex
\subsection{Problem Formulation}\label{back:mfe}
In a game with a large number of players, we might expect that the distribution of agents over the action space carries more meaning than the actions themselves. It is intuitive that a single agent has negligible effect on the game as the number of agents increases. The effect of other agents on a single agent's payoff is only via the action distribution of the population. This intuition is formalized as the {\em mean field game}. 

Let $\alpha_{-i,t}(a):\mathcal{A}\to[0,1]$ be the fraction of the agents (excluding agent $i$) that take action $a \in \mathcal{A}$ at time $t$. Mathematically, we have
\begin{align}
\alpha_{-i,t}(a)=\dfrac{1}{n-1}\sum_{m\in [n]/\{i\}}\mathbbm{1}(a_{m,t}=a),
\end{align}
where $\mathbbm{1}(a_{j,t}=a)$ is the indicator function that the agent $j$ takes action $a$ at time $t$. Further, since each agents selects exactly and only one action from $\mathcal{A}$, we have 
$$\alpha_{-i,t}(a)\geq 0\ \forall\ a\in\mathcal{A}\text{, and } \sum_{a\in\mathcal{A}}\alpha_{-i, t}(a) = 1$$

Similar to distribution of agents over actions, we define $f_{-i, t}:\mathcal{S}\to[0,1]$ as the distribution of agents (excluding agent $i$) over the state space $\mathcal{S}$.
\begin{align}
f_{-i, t}(a)=\dfrac{1}{n-1}\sum_{m\in [n]/\{i\}}\mathbbm{1}(s_{m,t}=s),
\end{align}

In a mean field game, every agent $i\in N$ conjectures that the reward and the next state is randomly distributed according to the transition probability distribution $P_i$ conditioned on agent's current state and action and other agents' distribution over actions and states.
\begin{align}
s_{i,t+1}&\sim p_i(\cdot|s_{i,t},a_{i,t},\alpha_{-i,t})\\
R_{i, t} &= r_i(s_{i,t},a_{i,t}, s_{i, t+1})
\end{align}
Thus the agent doesn't need to concern itself with the action of the other agents, while the average of the population action using $\alpha_{-i,t}$ becomes a part of the environment which plays a role in the decision. The average of the population action $\alpha_{-i,t}$ could be explicitly taken into account for decising an action.  However, Proposition 1 from \cite{johari} says that under equilibrium, an oblivious strategy performs as good as a strategy which considers other agents' actions. Thus, the stratgey does not need to explicitly use the value of $\alpha_{-i,t}$.

\begin{definition}
An agent $i\in[n]$ is said to follow oblivious deterministic strategy when the agent makes selects an action considering only current state and episode step.
\begin{align}
\pi_i : \mathcal{S} \times [\tau] &\rightarrow \mathcal{A}\\
a_{i, j} &= \pi_i(s_{i,j}, j)
\end{align}
\end{definition}

For the rest of the paper, we will focus on oblivious deterministic strategy for all agents. We now define a value function for an agent $i\in[n]$ for oblivious policy $\pi_i$ at $l^{th}$ time step as:
\begin{equation}
\begin{aligned}
    &V_{i, \pi_i,l}(s|\alpha_{-i,l}) = \\
    &\mathbf{E}_{\pi_i}\left[\sum_{j=l}^{\tau-1}\gamma^{j-l}r_i(s_{i, j}, a_{i,j}, s_{i, j+1})|s_{i, l} = s\right]. \label{eq:val_fn}
\end{aligned}    
\end{equation}

We will consider the rest of definitions from some $i^{th}\in N$ agent's perspective, so subscripts $i$ and $-i$ will be dropped for brevity.




We note that, the action space and state space are finite and hence the set of strategies available to the players is also finite. The player adopts the lower myopic best response dynamics to choose the policy. A lower myopic policy selects an action with the lowest index  among the actions that maximise the value function.  As time proceeds, the strategies and the action distribution converge to the asymptotic equilibrium \cite{johari}. 

Let $\alpha^*_j$ be the limiting action distribution for $j^{th}$ index in episode $k$. Then from the definition of limit, for every $\epsilon > 0$ there exist a $K_\epsilon< \infty$ such that for all $k> K_\epsilon$, we have
\begin{align}
    |\alpha_{(k-1)\tau + j} - \alpha^*_{j}| < \epsilon
\end{align}


The value function defined in Equation \ref{eq:val_fn} satisfies the Bellman-property for finite horizon MDPs, given by
\begin{align}
    V_{\pi,l}(s|\alpha^*_l) &= \mathbf{E}_{\pi_i}\left[\sum_{j=l}^{\tau-1}\gamma^{j-l}r_i(s_{i, j}, a_{i,j}, s_{i, j+1})|s_{i, l} = s\right]\\
    &= \sum_{s'\in S}p(s'|s,a,\alpha^*_l, f)r_i(s_{i,l},a_{i,l}, s_{i, l+1})+ \nonumber\\ &\mathbf{E}_{\pi_i}\left[\sum_{j=l+1}^{\tau-1}\gamma^{j-l}r_i(s_{i, j}, a_{i,j}, s_{i, j+1})|s_{i, l+1} = s'\right]\\
    &=\Bar{r}(s, a, \alpha^*_l) + \gamma\sum_{s'\in S}p(s'|s,a,\alpha^*_l)V_{\pi,l+1}(s'|\alpha^*_l)
\end{align}

Where $\Bar{r}(s, a, \alpha^*_l) = \sum_{s'\in S}p(s'|s,a,\alpha^*_l, f)r_i(s, a, s')$, and $a = \pi(s,l)$.

Similarly, we also define the $Q$-function as:
\begin{align}
    Q_{\pi,l}(s,a|\alpha_l) = \Bar{r}(s, a, \alpha^*_l) + \gamma\sum_{s'\in S}p(s'|s,a,\alpha_l)V_{\pi,l+1}(s'|\alpha^*_l)
\end{align}

We further consider the agents are strategic or selfish. The goal of each agent is to find an optimal oblivious policy $\pi$, such that,
\begin{align}
    &V_{\pi^*,l}(s|\alpha^*_{l}) \geq V_{\pi,l}(s|\alpha^*_{l})
    &\ \forall s\in\mathcal{S},\ \forall l\in[\tau].
\end{align}

We, now, define the optimal oblivious strategy:
\begin{definition}
The set $\mathcal{P}(\alpha^*)$ is the set of the optimal oblivious strategies which are chosen from the $Q$-function generated by $\alpha^*$. In other words, for a given $\alpha^*$, a policy ${\bar \pi} \in \mathcal{P}(\alpha^*)$ if and only if
\begin{align}
\bar{\pi}(s,j) = \arg\max_a Q_{\bar{\pi},j}(s,a|\alpha^*) \forall s\in S\ j \in [\tau]
\end{align}
\end{definition}

Here, the policy $\bar{\pi}(s,j)$ is used at $j^{th}$ time step so that the $Q$-value  $Q_{\pi,j}(s)$ is maximized for all states $s \in \mathcal{S}$. {\em Note that $\bar{\pi}(s,j)$ does not depend on the distribution $\alpha^*$ explicitly. Hence, it is an oblivious strategy where each agent takes its decision based on its own observed state only.}
The set $\mathcal{P}(\alpha^*)$ can be empty however in the subsequent section we'll show that the set $\mathcal{P}(\alpha^*)$ is non-empty under the assumptions used in this paper. 
We denote $\pi(s,j)$ as the strategy which has been learnt till episode $j$.

Suppose the agents play  optimal oblivious strategy $\Bar{\pi} \in \mathcal{P}(\alpha^*)$. The initial population state distribution denoted by $f$ evolves with all the agents converging to the limiting action distribution $\alpha^*$ then if the long run state distribution is equal to the initial state distribution $f$, the distribution $f$ is said be invariant of the dynamics induced by $\alpha^*$ and $\Bar{\pi} $. We denote the set of all such state distributions through a map $\mathcal{D}(\Bar{\pi} ,\alpha^*)$.


We assume that no agent  observes the state of the other agents. Thus, {\em an agent does not know the probability transition matrix} and will try to estimate it from the past observations as described in the next section. 



%% file: algo.tex
\section{Proposed Algorithm}\label{pos_samp}

In this section, we propose an algorithm, which will be shown to converge to the mean field equilibrium (MFE) in the following section. 
For each agent $i$, the algorithm begins with a prior distribution $g$ over the stochastic game with state space set $S$ and action space set $A$ and time horizon $\tau$. The game is played episodes $k = 0, 1, 2,\cdots$. The length of each episode is given by $\tau$. In each episode, the game is played in discrete steps, $j = 1, 2, \cdots, \tau$. The episodes begin at times $t_k = (k-1)\tau + 1, k=1, 2, 3, \cdots$. At each time $t = (k-1)\tau + j$, the state of the agent is given by $s_{t}$, it selects an action $a_{t}$, and observes a scalar reward $r_{t}$ then transitions to the state $s_{t+1}$. Let $H_{t} = (s_{1},a_{1}, r_{1}, \dots s_{t}, a_{t-1}, r_{t-1})$ denote the history of the agent till time $t$.



The proposed algorithm is described in Algorithm \ref{alg:main}. At the beginning of each episode, the MDP, ${\cal M}_{k}$ is sampled from the posterior distribution conditioned on the history $H_{t_k}$ in Line 4. We note that the sampling of MDP only relates to the sampling of the transition probability $P$ and the reward distribution since the rest of the parameters are known. We assume that after some samples, $\alpha_{k}$ has converged. The proposed algorithm converges after $\alpha_{k}$ and the induced transition probability and reward function converge. 

We use Backwards Induction algorithm \cite{puterman2014markov} described in Algorithm \ref{alg:helper} to obtain the Q-value function for the current sampled MDP (Line 5, Algorithm \ref{alg:main}). Backwards induction in Algorithm \ref{alg:helper} starts from the end of the episode and calculates the potential maximum rewards for each state and action (Line 5). The algorithm than goes back in the episode (Line 8), to calculate the maximum possible cumulative rewards for each state and action in Line 11. After all the time indices in an episode are covered, the algorithm returns the calculated optimal Q-values. We obtain the policy $\pi_k$ from the the calculated Q-values and the policy is not altered in an episode. Recall for a given $\alpha$, a policy $\pi \in \mathcal{P}(\alpha)$ if and only if $\pi_k(s,j) = \arg\max_aQ_{\pi_k,j}(s,a|\alpha)$  for all $s \in \mathcal{S}$ and $j = 0, 1, \dots, \tau-1$. In order to choose the policy $\pi_k$ form the set $\mathcal{P}(\alpha_{k})$, we use lower myopic learning dynamics, where at each episode we choose the strategy which is smallest action index in the set $\mathcal{P}(\alpha_{k})$.



\begin{algorithm}[htbp]
\caption{Proposed Algorithm for Mean Field Game with Best Response Learning Dynamics}\label{alg:main}
\begin{algorithmic}[1]
    \State \textbf{Input:} Prior distribution $g$, time horizon $\tau$, $\gamma$
	\State Initialize $H_0 = \phi$. 
	\For{episodes $k=1, 2, \cdots$ }
	\State Sample $\mathcal{M}_{k} \sim g(\cdot|H_{(k-1)\tau+1})$ . 
	\State Obtain optimal $Q$ for $\mathcal{M}_k$ from Algorithm 2
	\For{time steps $j = 1$, $\dots$, $\tau$}
		    \State Play $a_j = \arg\max_a Q_j(s_j, a)$.
			\State Observe reward $r_j$, action of the agent $a_j$, and next state $s_{j+1}$.
			\State Append action taken $a_j$, reward obtained $r_j$, and state update $s_{j+1}$ to history
			$$H_{(k-1)\tau+j+1} = H_{(k-1)\tau+j}\cup \{a_j, r_j, s_{j+1}\}.$$
\EndFor
\EndFor
\end{algorithmic}
\end{algorithm}

We note that ${\cal M}_{k}$ is used in the algorithm instead of ${\mathcal{M}^*}$ where $\cal M^*$, the true distribution, is not known. In order to obtain an estimate, each agent samples a transition probability matrix according to the posterior distribution. Each agent follows the strategy $\pi_k$ according to the $Q$-values over the episode. Based on action decision by each agent, we update the value function and the $Q$-function based on the obtained reward functions which depend on the value of $\alpha_{k}$. The detailed algorithm steps can be seen in Algorithm  \ref{alg:main}. We note that, after the algorithm converges, the value of $\alpha$ converges and thus all the transition probabilities and value functions depend on the limiting distribution.

\begin{algorithm}[H]
\caption{Backwards Induction Algorithm}\label{alg:helper}
\begin{algorithmic}[1]
    \State \textbf{Input:} $\mathcal{M}=\{\mathcal{S}, \mathcal{A}, P, r, \tau, \gamma\}$ \Comment{Sampled MDP from Algorithm \ref{alg:main}}
	\State Initialize $Q_l(s, a) = 0\ \forall\ s\in \mathcal{S}, a\in \mathcal{A}, l\in [\tau]$. 
    \For {state $s \in \mathcal{S}$}
        \For {state $a \in \mathcal{A}$}
            \State Update Q-value function for last action
            $$Q_\tau(s, a) = \sum_{s'\in \mathcal{S}}P(s'|s, a)r(s, a, s')$$
        \EndFor
    \EndFor
	\For{time steps $l = \tau-1$, $\dots$, $1$}
	    \For {state $s \in \mathcal{S}$}
	        \For {state $a \in \mathcal{A}$}
                \State Update Q-value function
                \begin{align}
                    Q_l(s, a) &= \sum_{s'\in \mathcal{S}}P(s'|s, a)\times\nonumber\\
                    &\left(r(s, a, s')+ \arg\max_a Q_{l+1}(s', a)\right)
                \end{align}
            \EndFor
		\EndFor
    \EndFor
    \State {\bf Return:} $Q_l(s,a)\ \forall\ l, s, a$
\end{algorithmic}
\end{algorithm}

%% file: results.tex
\section{Convergence Result}
\newcommand{\expc}[1]{{\mathbb E}\left[ #1 \right]}

In this section, we'll show that if the oblivious strategy is chosen according to the proposed algorithm, then the oblivious strategy $\pi$ and the limiting population action distribution $\alpha$ constitutes a Mean Field Equilibrium (MFE). More formally, we have

\begin{theorem}
The optimal oblivious strategy obtained from the Algorithm 1 and the limiting action distribution constitute a mean-field equilibrium and the value function obtained from the algorithm converges to the optimal value function of the true distribution.
\end{theorem}

The rest of the section proves this result. We first note that the lower-myopic best response strategy leads to a convergence of the action strategy following the results in \cite{johari} for finite action space and state space. The key intuition for the lower-myopic strategy is to avoid conflicts when there is non-unique strategy at the agent that maximizes the value function which might lead to choosing different strategies at different iterations if lower myopic is not used. Further, any way of resolving the multiple optimas could be used, including upper-myopic giving the same result. Having shown that $\alpha$ converges, we now proceed to show that the converged point of the algorithm results in a MFE.

We first show the  conditions needed for a policy $\pi$,  a population state $f$, and action distribution $\alpha$ to constitute an MFE (Section \ref{cond_mfe}). Then, we show that the conditions for the policy to be MFE given in Section \ref{cond_mfe} are met for any optimal oblivious strategy (Section \ref{sec_obliv}). Thus, the key property that is required to show the desired result is that the proposed algorithm leads to an optimal oblivious strategy. In order to show that, we first show that the optimal oblivious strategy set is non-empty (Section \ref{sec_nonempty}). Then, we show that the value function of the sampled distribution converges to the true distribution (Section \ref{no_loss_sampling}). The result in Section \ref{no_loss_sampling} shows that the update of the value function steps eventually converge to the value function with knowledge of true underlying distribution of the transition probability ${\cal M}^*$, thus proving that the proposed algorithm converges to an optimal oblivious strategy which constitute a mean field equilibrium  thus proving the theorem. 


\subsection{Conditions for a Strategy to be a MFE}\label{cond_mfe}

In this section, we will describe the conditions for an oblivious strategy $\pi$ to be a MFE.  In Section \ref{back:mfe}, we defined two maps $\mathcal{P}(\alpha)$ and $\mathcal{D}(\mu,\alpha)$. For a given action coupled stochastic game, the map $\mathcal{P}(\alpha)$ for a given population action distribution $\alpha$ gives the set of the optimal oblivious strategies. Further,  the map $\mathcal{D}(\mu,\alpha)$ for a given population action distribution $\alpha$ and oblivious strategy $\mu$ gives the set of invariant population state distribution \textit{f}. 

We define the map $\mathcal{\hat{D}}(\mu,\alpha)$ which gives the induced population distribution $\alpha$ induced from the oblivious strategy $\mu$ and the population state distribution \textit{f}. The following lemma gives the conditions that the stochastic game constitutes a mean field equilibrium. These conditions have been provided in \cite{johari}, and the reader is referred to \cite{johari} for further details and proof of this result. 

\begin{lem}\label{lem_condition}
An action coupled stochastic game with the strategy $\pi$, population state distribution \textit{f} and population action distribution $\alpha$ constitute a mean field equilibrium if $\pi \in \mathcal{P}(\alpha)$, \textit{f} $\in \mathcal{D}(\pi,\alpha)$ and $\alpha = \mathcal{\hat{D}}(\pi,\textit{f})$.
 \end{lem}

\subsection{Conditions of Lemma \ref{lem_condition} are met for any Optimal  Oblivious Strategy }\label{sec_obliv}

In this section, we show that the conditions of Lemma \ref{lem_condition} are met for any optimal  oblivious strategy.  In the mean field equilibrium, each agent plays according to the strategy $\pi \in \mathcal{P}(\alpha)$. If the long run average population action distribution is $\alpha$, and each agent takes an oblivious strategy, hence, we must have the evolution of the state space such that the oblivious strategy on those states leads to an average action distribution of $\alpha$. Let the long run average state distribution be $f$, {\em i.e.},
\begin{align}
f(s)=\dfrac{1}{n}\sum_{i}\mathbbm{1}_{s_i=s}
\end{align}
where $s_i$ is the state of the agent $i$.  Then the above statement implies that $\alpha$ must satisfy
\begin{align}\label{eq:alph}
\alpha(a)=\sum_{\pi^{-1}(a)}f(s)
\end{align}
where $\pi^{-1}(x)$ represents the set of states for which $a \in \pi(x)$. This is equivalent to saying that if all the agents follow the optimal oblivious strategy $\pi \in \mathcal{P}(\alpha)$, then the long run average population state distribution \textit{f} and the long run average population action $\alpha$ satisfy \textit{f} $\in \mathcal{D}(\pi,\alpha)$ and $\alpha = \mathcal{\hat{D}}(\pi,f)$.

\subsection{Optimal Oblivious Strategy Set is Non-Empty}\label{sec_nonempty}
In this subsection, we show that there exists an optimal oblivious strategy. More formally, we have the following lemma. 
\begin{lem}
For the limiting population average action distribution $\alpha$, the set of the optimal oblivious strategies given by $\mathcal{P}(\alpha)$ is non-empty. A policy $\pi(s) \in \mathcal{P}(\alpha)$ if,
\begin{align}\label{eq:max}
\pi(s,j)\in \arg&\max_{a\in A} \{\Bar{r}(s, a, \alpha) +\nonumber\\
&\gamma\sum_{s'\in S}p(s'|s,a,\alpha)V_{\pi,j+1}(s'|\alpha)\}
\end{align}\label{lem:pol}
\end{lem}
\begin{proof} Note that $\bar{r}$ is bounded and Lipschitz continuous. In addition, for each state $s$, the next state is drawn from a countable set. Further, 
\begin{align}
    V_{\pi,\ell}(s|\alpha) = \mathbf{E}_{\pi}\left[\sum_{j=\ell}^{\tau-1}\gamma^j\Bar{r}(s_{j}, \pi(s_{j}), \alpha)|s_\ell = s\right]
\end{align}
An oblivious strategy is optimal if and only if it attains a maximum on the right hand side of Eq. (\ref{eq:max}) for every $s$. Since the reward is bounded, and $\gamma<1$, thus, $V_{\pi,\ell}$ is bounded for all $\ell = 0, 1, \dots, \tau-1$ which means there exists an optimal oblivious strategy that maximizes the RHS. Hence, the set $\mathcal{P}(\alpha)$ for a given $\alpha$ is non-empty.
Therefore for each episode, there exists an optimal oblivious strategy which is given by $\pi_k \in \mathcal{P}(\alpha_{k})$.
\end{proof}

\subsection{Sampling does not Lead to a Gap for Expected Value Function} \label{no_loss_sampling}
In the last  subsection, we proved that there exists an optimal oblivious strategy. In this subsection, we'll show that  the expected optimal value function achieved by the algorithm and the true distribution is equal.  We first describe the Azuma-Hoeffding Lemma that will be used in the result. 

\begin{lem}(Azuma-Hoeffding Lemma \cite{osband2013more}) If $ Y_{n} $ is a zero-mean martingale with almost surely bounded increments, $|Y_{i}$ -  $Y_{i-1}|$ $\leq$ C, then for any $\delta$ $\geq$ 0 with probability at least 1-  $\delta$, $Y_{n} \leq C\sqrt{2n\log(1/\delta)}$. \label{lem:azuma}
\end{lem}

At the start of every episode, each agent samples a probability distribution from the posterior distribution. The following result bounds the difference between the optimal value function learned by the true distribution ${\cal M}^{*}$ function using the optimal policy $\pi^{*}$ which is unknown, and the optimal value function achieved by the sampled distribution ${\cal M}_k$ from the  policy $\pi_{k}$.

\begin{lem}\label{thm:dlta}
Let $V^{{\cal M}_{k}}_{\pi_{k},j}(s|\alpha))$ be the optimal value function of the sampled distribution $\mathcal{M}_k$ chosen form Algorithm 1. Then $V^{{\cal M}_{k}}_{\pi_{k},j}(s|\alpha))$ converges to the optimal value function of the true distribution $\mathcal{M}^*$, $V_{\pi^{*},j}(s|\alpha)$, {\em i.e.}, for all states $s \in \mathcal{S}$ as $k \rightarrow \infty$,
\begin{align}
    V^{{\cal M}_{k}}_{\pi_{k},j}(s|\alpha)) - V_{\pi^{*},j}(s|\alpha) \rightarrow 0
\end{align}

\end{lem}
\begin{proof}


To prove this, we first show an equivalence of the true distribution and the sampled distribution which comes from the property of  posterior sampling shown in \cite{osband2013more}, which says that for any $\sigma-$ measurable function of history $H_{t_k}$, we have 
\begin{align}
\mathbf{E}[g(\mathcal{M}_k)] = \mathbf{E}[g(\mathcal{M}^*)]
\end{align}

which can be applied to the difference of the optimal functions of the two distributions to show that for all states $s \in \mathcal{S}$,
\begin{align}
     \mathbf{E}\left[V_{\pi^{*},j}(s|\alpha) - V^{{\cal M}_{k}}_{\pi_{k},j}(s|\alpha))\right] = 0
\end{align}
Note that the length of all episodes is given by $\tau$ and the support of the reward is [0,1]. Therefore for all states $s \in \mathcal{S}$, we have $V_{\pi^{*},j}(s|\alpha) - V^{{\cal M}_{k}}_{\pi_{k},j}(s|\alpha)) \in [-\tau,\tau]$. Note that this condition is  similar to bounded increments in Azuma-Hoeffding Lemma (Lemma \ref{lem:azuma}). 

Since $V_{\pi^{*},j}(s|\alpha) - V^{{\cal M}_{k}}_{\pi_{k},j}(s|\alpha) \in [-\tau,\tau]$ is a zero mean martingale with respect to the filtration  $\{ H_{t_{k}} : k= 1,..,m \}$, and satisfies the assumptions of Azuma-Hoeffding Lemma, we obtain the result as in the statement of the Lemma. 
Also, for all states $s \in \mathcal{S}$, we have, $V_{\pi^{*},j}(s|\alpha) - V^{{\cal M}_{k}}_{\pi_{k},j}(s|\alpha)) \in [-\tau,\tau]$. So, the difference is a zero-mean martingale and has the bounded increments property. Applying the Azuma-Hoeffding Lemma to the martingale, we have the following result,
\begin{align}
    \sum_{k=1}^{m}V_{\pi^{*},j}(s|\alpha) - V^{{\cal M}_{k}}_{\pi_{k},j}(s|\alpha) \leq \tau\sqrt{2m\log(1/\delta)}
\end{align}

For total time $T$ of the algorithm, we have $m = T/\tau$. Thus, for all $\theta > 1/2$, as $T \rightarrow \infty$, we have


\begin{align}
    \frac{\sum_{k=1}^{\left \lceil{T/\tau}\right \rceil}V_{\pi^{*},j}(s|\alpha) - V^{{\cal M}_{k}}_{\pi_{k},j}(s|\alpha)}{T^\theta} \rightarrow 0
\end{align}

Substituting $\theta = 1$, the above expression says that $\tau$ times the average difference in an episode which converges to zero as total time $T \rightarrow \infty$, which gives us the convergence of the optimal value functions of the two distributions. Thus, we have

\begin{align}
    V_{\pi^{*},j}(s|\alpha) - V^{{\cal M}_{k}}_{\pi_{k},j}(s|\alpha) \rightarrow 0 \text{ as } k \rightarrow \infty
\end{align}
\end{proof}

%% file: conclusion.tex
\section{Conclusion}
We consider an action coupled stochastic game consisting of large number of agents where the transition probabilities are unknown to the agents. We resort to the concept of mean-field equilibrium where each agent's reward and the transition probability is only impacted through the mean distribution of the actions of the other agents. When the number of agents grows large, the mean-field equilibrium becomes equivalent to the Nash equilibrium. We propose a posterior sampling based approach where each agent draws a sample using an updated posterior distribution and selects an optimal oblivious strategy accordingly. We show that the proposed algorithm converges to the mean field equilibrium without knowing the transition probabilities apriori. 

This paper shows asymptotic convergence to the mean-field equilibrium, while finding the convergence rate is an interesting future direction. 